\documentclass{article}
\usepackage{arxiv}

\usepackage{graphicx}
\usepackage{amsthm}
\usepackage{amssymb}
\usepackage{amsmath}

\usepackage{booktabs}
\usepackage{multirow}
\usepackage{multicol}
\usepackage{paralist}
\usepackage{enumitem}
\usepackage[round]{natbib}

\usepackage[caption=false]{subfig}
\usepackage{rotating}

\usepackage{hyperref}
\usepackage{url}

\usepackage{color}

\usepackage[linesnumbered,algoruled,boxed,lined,norelsize]{algorithm2e}

\usepackage{bm}

\newcommand{\R}{\mathbb{R}}

\newcommand{\vdelta}{{\bm{\delta}}}
\newcommand{\vlambda}{{\bm{\lambda}}}
\newcommand{\valpha}{{\bm{\alpha}}}

\newcommand{\va}{{\bm{a}}}
\newcommand{\vb}{{\bm{b}}}

\newcommand{\ve}{{\bm{e}}}

\newcommand{\vg}{{\bm{g}}}

\newcommand{\vu}{{\bm{u}}}
\newcommand{\vw}{{\bm{w}}}
\newcommand{\vx}{{\bm{x}}}

\newcommand{\mA}{{\bm{A}}}

\newcommand{\mS}{{\bm{S}}}

\newcommand{\mU}{{\bm{U}}}
\newcommand{\mV}{{\bm{V}}}

\newcommand{\mSigma}{{\bm{\Sigma}}}

\newcommand{\sD}{{\mathbb{D}}}

\newcommand{\sN}{{\mathbb{N}}}

\newcommand{\sR}{{\mathbb{R}}}
\newcommand{\sS}{{\mathbb{S}}}
\newcommand{\sT}{{\mathbb{T}}}

\newcommand{\sV}{{\mathbb{V}}}

\newcommand{\sX}{{\mathbb{X}}}
\newcommand{\sY}{{\mathbb{Y}}}

\DeclareMathOperator*{\argmax}{arg\,max}
\DeclareMathOperator*{\argmin}{arg\,min}

\newtheorem{definition}{Definition}
\newtheorem{theorem}{Theorem}
\newtheorem{lemma}{Lemma}
\newtheorem{corollary}{Corollary}

\begin{document}

\title{Spanning attack: reinforce black-box attacks \\with unlabeled data}





\author{Lu Wang\textsuperscript{\textnormal{1,2}}
\quad Huan Zhang\textsuperscript{\textnormal{3}}
\quad Jinfeng Yi\textsuperscript{\textnormal{2}}
\quad Cho-Jui Hsieh\textsuperscript{\textnormal{3}}
\quad Yuan Jiang\textsuperscript{\textnormal{1}}
\\
\textsuperscript{1}National Key Laboratory for Novel Software Technology, \\
Nanjing University, Nanjing 210023, China\\
\textsuperscript{2}JD AI Research, JD.com, Beijing 100101, China\\
\textsuperscript{3}Department of Computer Science,
University of California, Los Angeles, CA 90095\\
\texttt{wangl@lamda.nju.edu.cn} / \texttt{wangl@jd.com}
\quad \texttt{huanzhang@ucla.edu} \\
\quad \texttt{yijinfeng@jd.com}
\quad \texttt{chohsieh@cs.ucla.edu}
\quad \texttt{jiangy@lamda.nju.edu.cn}
}

\maketitle


\begin{abstract}
    Adversarial black-box attacks aim to craft adversarial perturbations
    by querying input--output pairs of machine learning models.
    They are widely used to evaluate the robustness of pre-trained models.
    However, black-box attacks often suffer from the issue of query inefficiency due to the high dimensionality of the input space,
    and therefore 
    incur a false sense of model robustness.
    In this paper, we relax the conditions of the black-box threat model,
    and propose a novel technique called the \emph{spanning attack}. 
    By constraining adversarial perturbations in a low-dimensional subspace
    via \emph{spanning} an auxiliary unlabeled dataset,
    the spanning attack significantly improves the query efficiency
    of a wide variety of existing black-box attacks.
    Extensive experiments show that the proposed method works favorably
    in both soft-label and hard-label black-box attacks.
    Our code is available at \url{https://github.com/wangwllu/spanning_attack}.
\end{abstract}



\section{Introduction}

It has been shown that machine learning models,
especially deep neural networks,
are vulnerable to small adversarial perturbations,
i.e., a small carefully crafted perturbation
added to the input may significantly change 
the prediction results~\citep{szegedy2014intriguing,goodfellow2015explaining,biggio2018wild,fawzi2018analysis}. 
Consequently, the problem of finding those perturbations,
also known as \emph{adversarial attacks},
has become an important way to evaluate model robustness:
the more difficult to attack a model, the more robust the model is. 

Depending on the type of information available to the adversary,
adversarial attacks can be categorized into
\emph{white-box attacks} and \emph{black-box attacks}.
In the {white-box} setting, 
the \emph{target model} (the model to attack) is completely exposed to the attacker, 
and adversarial perturbations could be
crafted by exploiting the first-order information
(or any higher order information),
i.e., gradients
with respect to the input~\citep{carlini2017towards,madry2018towards}. 
Despite its efficiency and effectiveness,
the white-box setting often stands for an overly strong and pessimistic \emph{threat model},
and white-box attacks are usually not practical
when attacking real-world machine learning systems
due to the fact that their gradient information is often invisible to the attacker.

In this paper, we focus on the problem of \emph{black-box} attacks: 
the case where the model structure and parameters (weights)
are not available to the attacker~\citep{chen2017zoo}.
The attacker can only gather necessary information
by means of (iteratively) making input queries to the model
and obtaining the corresponding outputs.
The black-box setting is a more realistic threat model,
and furthermore, crucial in the sense that
they could serve as a general way to evaluate the robustness
of machine learning models beyond neural networks,
even when the model is not differentiable
(e.g., evaluating the robustness of tree-based  models~\citep{chen2019robust}
and nearest neighbor models~\citep{wang2019evaluating,wang2020provably}).

Black-box attacks have been extensively studied in the past few years.
Depending on what kind of outputs the attacker could derive, black-box attacks could be
broadly grouped into two categories: 
\emph{soft-label attacks}~\citep{chen2017zoo}
and \emph{hard-label attacks}~\citep{brendel2018decision}. 
Soft-label attacks assume that the attacker
has access to real-valued scores (logits or probabilities) for all labels,
while hard-label attacks assume  that the attacker only has access to
the final discrete decision (the predicted label). 
However, black-box attacks,
especially hard-label attacks,
usually require a large number of (typically $>10K$) queries
for each adversarial perturbation.
High query complexity
limits the scope of application of black-box attacks,
and also incurs a false sense of model robustness.

We notice that the convergence rates
of the zeroth-order optimization methods
used for black-box attacks
are shown to be proportional to the dimensionality of the input space~\citep{nesterov2017random,wang2017stochastic,tu2019autozoom}. 
As a consequence,
we have a natural conjecture:
the query complexity of black-box attacks
is also dependent on the dimensionality of the input space,
and thus reducing its dimensionality \emph{in a certain delicate way} can 
enhance the query efficiency of black-box attacks. 

Based on the idea above, in this paper we propose a method
--- the \emph{spanning attack} ---
to constrain the search space of black-box attacks
for the purpose of tackling the inefficiency issue.
The spanning attack is motivated by our theoretical analysis that
\emph{minimum adversarial perturbations} for a variety of machine learning models
prove to be in the \emph{subspace} of the training data.
Specifically,
we relax the conditions of the black-box threat model by additionally assuming that
a small \emph{auxiliary unlabeled dataset} is available to the attacker.
The assumption is reasonable:
imagine that before attacking an image classification model,
the attacker just needs to ``collect'' some unlabeled images,
from the Internet for instance.
This auxiliary unlabeled dataset plays as a substitute for the original training data:
this dataset \emph{spans} a subspace of the input space.
Then, we constrain the search of adversarial perturbations only in this subspace,
of which the dimensionality is much smaller than the one of the original input space.
The overall workflow of the spanning attack is illustrated in Figure~\ref{fig:workflow}.

\begin{figure}
    \centering
    \includegraphics[width=.8\textwidth]{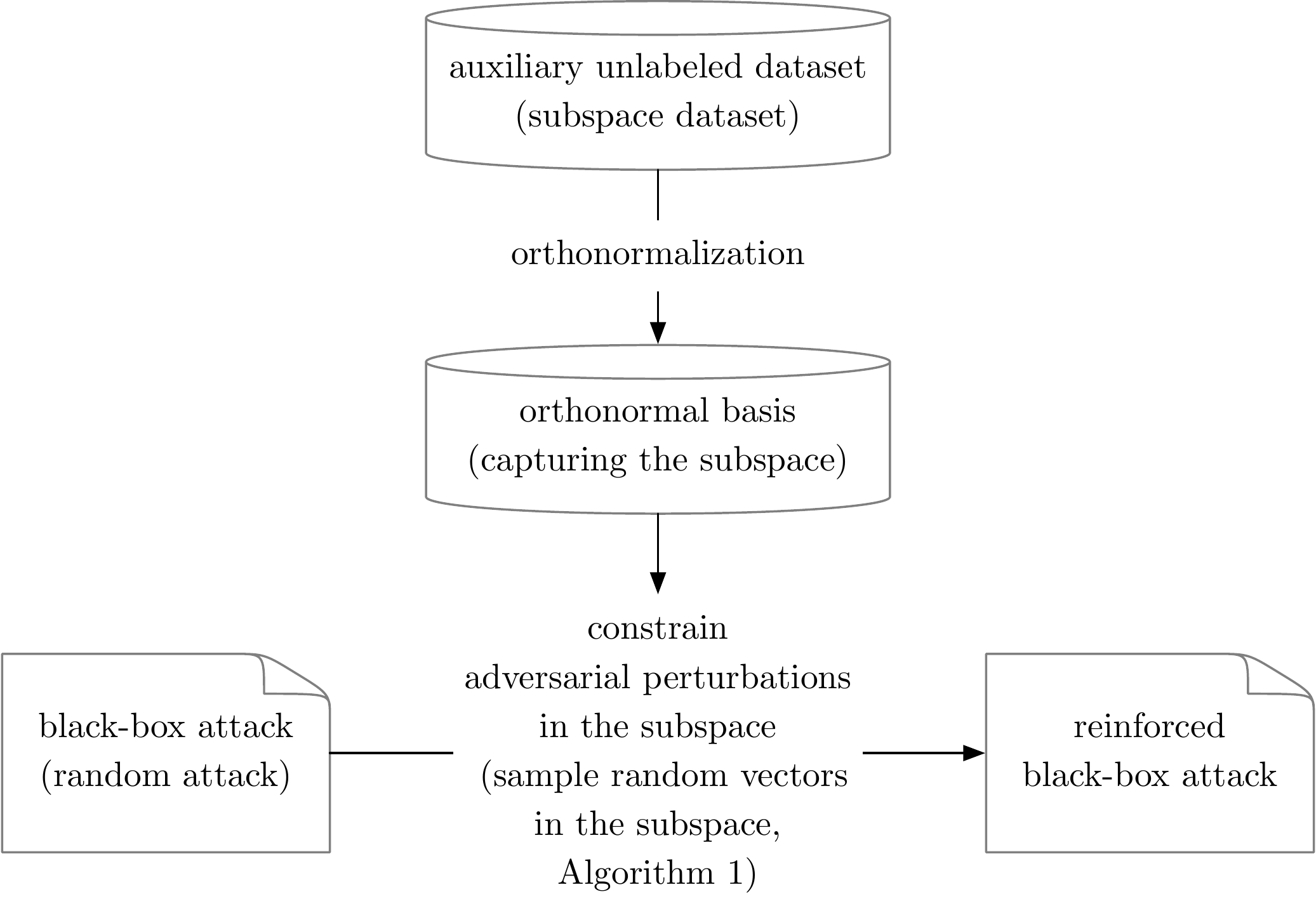}
    \caption{Workflow of the spanning attack.}
    \label{fig:workflow}
\end{figure}

We also show that the spanning attack method is general enough
to apply to a wide range of existing black-box attack methods,
including both soft-label attacks and hard-label attacks.
Our experiments verify that the spanning attack could
significantly improve query efficiency of black-box attacks.
Furthermore, we show that even a very small and \emph{biased} unlabeled dataset
sampled from a distribution different from the training data suffices
to perform favorably in practice.
This finding suggests that
the assumption of the spanning attack (about the auxiliary unlabeled dataset)
is not too strict to be satisfied.

In summary, this paper makes the following contributions:
\begin{enumerate}
    \item We present the \emph{random attack} framework which captures most existing black-box attacks
    in various settings
    including both soft-label attacks and hard-label attacks.
    It is a novel and intuitive interpretation
    on the mechanism of black-box attacks
    from the perspective of random vectors.
    \item We propose a method to regulate the resulting adversarial perturbation of any random attack
    to be constrained in a predefined subspace.
    This is a general method to reduce the dimensionality of the search space of black-box attacks.
    \item We make preliminary theoretical analysis about the subspace
    in which the \emph{minimum adversarial perturbation} is guaranteed to be placed.
    Motivated by our analysis, we propose to reinforce black-box attacks (random attacks)
    by means of constraining a subspace spanned by an auxiliary unlabeled dataset.
    In our experiments across various black-box attacks and target models,
    the reinforced attack typically requires less than $50\%$ queries
    while improves success rates in the meantime.
\end{enumerate}

The remainder of the paper is organized as follows:
Section~\ref{sec:related} discusses related work about black-box attacks;
Section~\ref{sec:background} introduces the basic preliminaries and our motivation;
Section~\ref{sec:method} presents our framework for black-box attacks
and proposes our general method to improve query efficiency;
Section~\ref{sec:exp} reports empirical evaluation results;
Section~\ref{sec:conclusion} concludes this paper.

\section{Related work} \label{sec:related}

\paragraph{Transfer-based black-box attacks.}
The first practical black-box attack is the transfer-based attack~\citep{papernot2017practical}.
A substitute model is trained with synthetic instances labeled by the target model (solf labels or hard labels).
Then, the adversarial perturbation is crafted to fool the target model by attacking the substitute model.
The effectiveness highly depends on
transferability of adversarial perturbations~\citep{papernot2016transferability,liu2017delving}.
Accordingly, the attack performance is severely
degraded with poor transferability~\citep{su2018is}. 
Therefore, we mainly talk about black-box attacks based on zeroth-order optimization as below.

\paragraph{Soft-label black-box attacks.}
\citet{chen2017zoo} showed that soft-label black-box attacks
can be formulated
as solving an optimization problem
in the zeroth-order scenario,
in which one can query the function itself but not its gradients.
Since then, many black-box attack methods based on zeroth-order optimization 
have been proposed
such as ZO-Adam~\citep{chen2017zoo},
NES~\citep{ilyas2018black},
ZO-SignSGD~\citep{liu2018signsgd},
AutoZOOM~\citep{tu2019autozoom}, and 
Bandit-attack~\citep{ilyas2019prior}.

\paragraph{Hard-label black-box attacks.}
Hard-label black-box attacks are more challenging
since it is non-trivial to define a smooth objective function for attacks
based only on the hard-label decisions.
\citet{brendel2018decision} proposed a method
based on reject sampling and random walks.
\citet{cheng2019query} reformulated the attack
as a real-valued optimization problem and the objective function
is estimated via coarse-grained search and then binary search.
\citet{chen2019hopskipjumpattack} proposed an unbiased estimator
of the gradient direction at the decision boundary,
and presented an attack method with a convergence analysis.
\citet{cheng2020signopt} proposed a query-efficient sign estimator of the gradient.

\paragraph{Improve query efficiency of black-box attacks.}
Recently, the idea of relaxing the threat model to
improve query efficiency of black-box attacks
has attracted increasing attention. 
Some work captured the idea of
transfer-based attacks~\citep{papernot2016transferability,liu2017delving}:
adversarial examples of a surrogate model also tend to fool other models.
\citet{brunner2018guessing} and \citet{cheng2019improving}
both assumed that a \emph{surrogate model} is available to the attacker.
Therefore, the attacker could employ the gradients of the surrogate model as a prior
for the true gradient of the target model.
Another work \citep{guo2019subspace}
proposed a soft-label black-box attack method
which employs an auxiliary \emph{labeled} datasets.
Multiple \emph{reference models} are trained with the \emph{labeled} datasets,
and a subspace is spanned by perturbed gradients of these reference models.
Then the true gradients of the target model are estimated in the subspace.
The major difference from our work is that their auxiliary dataset
has to be labeled, whereas ours is unlabeled.
Moreover, their auxiliary dataset is much larger
than ours owing to the need for training reference models:
in the ImageNet case, we only need less than 1,000 unlabeled instances,
whereas \citet{guo2019subspace} require 75,000 labeled instances.
Finally, our method is more general, and can be applied to both
soft-label and hard-label black-box attacks.

\section{Background and motivation} \label{sec:background}

We introduce the notations regarding black-box adversarial attacks.
Let $\sX = \sR^D$ denote the input space
where $D\in\sN^+$ is the number of dimensions,
and let $\sY = [C]$ denote the output space
where $C\in\sN^+$ is the number of labels.
The function $c: \sX \rightarrow \sY$ is a classifier (the target model)
and makes decisions by
\begin{align*}
    c(\vx) = \argmax_{i\in[C]}\ f(\vx)_i,
\end{align*}
where $f: \sX \rightarrow \R^C$ is the score function of the classifier,
which outputs scores of all labels for any given input.

Given a radius $\epsilon > 0$
and a correctly-classified labeled instance $(\vx, y)\in\sX\times\sY$,
the \emph{untargeted attack} aims to find an \emph{adversarial perturbation}
$\vdelta \in \sX$ with the norm $\lVert\vdelta\rVert\leq \epsilon$
such that the classifier predicts a label for the perturbed instance $\vx+\vdelta$
different from the original instance $\vx$, i.e., $c(\vx+\vdelta) \neq y$.
In comparison, the \emph{targeted attack} aims to make the classifier predict a pre-specified label.
Our paper will focus on untargeted attacks,
while it is easy to extend to targeted attacks.
Besides, we focus on the $\ell_2$ norm (the Euclidean norm) perturbation:
the magnitude of adversarial perturbations are measured by the $\ell_2$ norm,
and further research on general norms are deferred for future work.

In the soft-label setting,
the attacker has access to the score (logit or probability) output
for any input $\vx$ in $\sX$, i.e., $f(\vx)$.
Therefore, any loss function defined on the the pair of the score and the ground-truth label
is also available to the attacker.
We denote the loss function as $\ell_f(\vx, y)$.
In contrast,
in the hard-label setting,
the attacker only has access to the final decision (the predicted label)
for any input $\vx$ in $\sX$,
i.e., $c(\vx)$.
It is more challenging than soft-label attacks
due to less information available.
The number of queries,
to $f(\cdot)$ or $c(\cdot)$, is the cost of black-box attacks.
It is crucial to reduce the number of queries required
when applying attack methods in real applications.


In practice, the input space $\sX$ is usually high-dimensional:
for instance, the typical input image for an ImageNet model
has $224 \times 224 \times 3 = 150,528$ pixels.
It is suspected that requiring such a large amount of queries, often $>10K$,
when searching for an adversarial perturbation $\vdelta$ in $\sX$
is probably owing to the high dimensionality of $\sX$.
To verify our conjecture,
a natural question for black-box attacks is as below:
\begin{center}
    \emph{``Is it possible to reduce the number of queries \\
        for general black-box attacks \\
        by reducing the dimensionality of the search space?
        ''}
\end{center}
In this paper,
we provide a positive answer to this question
by proposing a method reinforcing black-box attacks with 
a small set of unlabeled data.

\section{Proposed method} \label{sec:method}

We first introduce the technique on
constraining (transforming) adversarial perturbations
into a predefined subspace
for general black-box attacks,
and then propose a method which utilizes an auxiliary unlabeled dataset
to select an appropriate subspace.

\subsection{Subspace transformation}

\begin{definition}[subspace attack]
    A subspace attack is an adversarial attack
    which returns adversarial perturbations
    in a predefined subspace $\sV\subseteq\sX$.
\end{definition}


Intuitively, the predefined subspace $\sV$ can be seen as a prior
for perturbations of adversarial examples.
If the subspace is small enough while still being able to capture most of small adversarial perturbations,
then due to the reduced dimensionality,
it can significantly reduce the number of queries required
for black-box attacks.

We will focus on ``one type'' of black-box attacks,
the \emph{random attack},
which captures a wide range of (nearly all existing) black-box attacks,
and is convenient to incorporate the prior knowledge about the subspace,
and thus easy to be transformed into a subspace attack.
Examples of random attacks will be shown
in Section~\ref{sec:case-soft} and Section~\ref{sec:case-hard}

\begin{definition}[random attack]
    The resulting adversarial perturbation of a random attack is a linear combination of random vectors.
\end{definition}

The following lemma highlights an intuition on how to transform a random attack into a subspace attack:
\begin{lemma}\label{lem:random-subspace}
    If all random vectors sampled by a random attack is constrained to be in a predefined subspace $\sV$,
    then the random attack is a subspace attack with respect to $\sV$.
\end{lemma}
The proof is straightforward: 
a linear combination of vectors in a subspace
is also in the subspace.


Random vectors of random attacks are typically sampled from isometric distributions:
all elements of the random vector are independent and identically distributed.
Typical examples of these distributions include the isometric Gaussian distribution
and the Rademacher distribution (uniform over $\{\pm 1\}$).
Let $\mathtt{sample}(d)$ denote the sampling routine
for such a random vector with the dimension $d\in\sN^+$.
(Thus $\mathtt{sample}(D)$ will sample a random vector in the original input space $\sX$.)
It follows that if we could constrain the sampling routine in a subspace,
by Lemma~\ref{lem:random-subspace} the resulting attack would be a subspace attack.
Specifically, Algorithm~\ref{alg:sample} displays how to sample a random vector in a subspace.
The subspace $\sV$ is characterized by an orthonormal basis
(see Section~\ref{sec:spanning} for details on how to derive the orthonormal basis),
and the term $\sqrt{\frac{D}{M}}$ guarantees that the returned random vector
has the same expected length as the original random vector $\mathtt{sample(D)}$.

\begin{algorithm}[tb]
    \SetAlgoLined
    \KwIn{orthonormal vectors $\ve_1, \ve_2,\ldots, \ve_M \in \sX$ which spans the subspace $\sV$,
        and the sampling routine \texttt{sample} for isometric random vectors}
    \KwOut{a random vector in the subspace $\sV$}
    $\vw \leftarrow \mathtt{sample}(M)$ \\
    \KwRet{$ \sqrt{\frac{D}{M}} \sum_{i=1}^M w_i \ve_i$}
    \caption{Random vectors in a subspace}
    \label{alg:sample}
\end{algorithm}

Note that the returned random vector of Algorithm~\ref{alg:sample} is a linear combination of the orthonormal vectors.
Therefore we have the following lemma:
\begin{lemma}
    The returned random vector of Algorithm~\ref{alg:sample}
    is constrained in the subspace $\sV = \operatorname{span}(\ve_1, \ve_2,\ldots, \ve_M)$,
    where $\operatorname{span}(\ve_1, \ve_2,\ldots, \ve_M)$
    returns the smallest space $\sV$ that contains
    all the input vectors $\ve_1$, $\ve_2$, $\ldots$, $\ve_M$.
\end{lemma}

Therefore, by applying Algorithm~\ref{alg:sample} to any random attack,
we have a subspace attack as the following corollary implies:
\begin{corollary} \label{cor:transform}
    Given a set of orthonormal vectors $\ve_1$, $\ve_2$, $\ldots$, $\ve_M \in \sX$,
    any random attack using isometric random vectors
    can be transformed into a subspace attack
    with the corresponding subspace
    $\sV = \operatorname{span}(\ve_1, \ve_2,\ldots, \ve_M)$
    by means of replacing the sampling routine \texttt{sample}($D$)
    via Algorithm~\ref{alg:sample}.
\end{corollary}

Corollary~\ref{cor:transform} introduces a particular method
to transform a random attack (black-box attack) into a subspace attack.
It is noteworthy that the transformation is performed 
only by means of replacing sampling routines.
It \emph{does not require to project adversarial perturbations
    from the input space $\sX$ to the subspace $\sV$
    explicitly},
and therefore causes as little negative impact as possible 
on the original random attack.

\subsubsection{Case study: soft-label black-box attacks} \label{sec:case-soft}

We investigate a soft-label black-box attack framework
within which the attack is composed of a gradient-based optimization method
and a \emph{backend} zeroth-order gradient estimation method.
This framework
is summarized in Algorithm~\ref{alg:soft},
and captures a wide range of soft-label black-box methods
\citep{ilyas2018black,liu2018signsgd,uesato2018adversarial,tu2019autozoom,cheng2019improving}.

\begin{algorithm}[th]
    \SetAlgoLined
    \KwIn{score function $f: \sX\rightarrow\sR^C$,
        and corresponding classifier $c:\sX\rightarrow\sY$,
        correctly-classified labeled instance $(\vx, y) \in \sX\times\sY$,
        and budget $B\in \sN^+$}
    \KwOut{adversarial perturbation $\vdelta$ or NULL}
    $\vdelta \leftarrow \texttt{initializer()}$ \\
    \While{$B > 0$}{
        \If{$c(\vx+\vdelta)\neq y$}{
            \KwRet{$\vdelta$} \tcp{successful}
        }
        $\vg \leftarrow \mathtt{gradient\_estimator}(f, \vx+\vdelta, y$) \\
        $\vdelta \leftarrow \mathtt{gradient\_based\_optimizer}(\vdelta, \vg)$ \\
        $B \leftarrow B - \mathtt{query\_sum\_within\_the\_iteration()}$
    }
    \KwRet{NULL} \tcp{failed}
    \caption{Soft-label black-box attack framework}
    \label{alg:soft}
\end{algorithm}

In this framework, random vectors could be introduced when
initializing the perturbation
(the all-zero vector or a random vector)
and estimating gradients
by the zeroth-order method.
A typical example of estimating gradients
is the random gradient-free (RGF) method~\citep{nesterov2017random},
which returns the estimated gradient in the form below:
\begin{align*}
    \hat{\vg} = \sum_i \frac{\ell_f(\vx + \sigma \vu_i, y) - \ell_f(\vx, y)}{\sigma} \vu_i,
\end{align*}
where $\vu_i$s are \emph{unit} Gaussian random vectors (Gaussian random vectors of length 1).
Therefore, $\hat{\vg}$ is a linear combination of random vectors.

Then, the resulting adversarial perturbation
is calculated by gradient-based optimization methods
such as projected gradient descent~\citep{madry2018towards},
all of which return linear combinations of the estimated gradients.
It follows that these attacks are random attacks and
could be easily transformed into a subspace attack via Algorithm~\ref{alg:sample}.

\subsubsection{Case study: hard-label black-box attacks} \label{sec:case-hard}

Hard-label black-box attacks could be separated into two categories:
methods based on random walks~\citep{brendel2018decision,chen2019hopskipjumpattack}
and methods based on direction estimation~\citep{cheng2019query,cheng2020signopt}.
In the first case, a random walk
consists of a succession of random vectors, i.e., the sum of random vectors;
in the second case, the gradient with respect to the direction towards the boundary
is estimated by RGF or its variant based on the sign of the finite difference.
As we discuss before, these gradient estimation methods typically return linear combinations of random vectors.
In both cases, the resulting adversarial perturbation is also a linear combination of random vectors,
and as a consequence they could also be transformed into subspace attacks obviously.

\subsection{Spanning attack}\label{sec:spanning}

The subspace $\sV$ is a prior for the subspace attack.
To make a subspace attack perform well,
it has to be easier to find an adversarial perturbation
in the subspace $\sV$ than in the original input space $\sX$.
The crux of the subspace attack is how to locate an appropriate subspace $\sV$.
We propose to utilize an auxiliary unlabeled dataset to span the subspace,
which is motivated by the theoretical analysis
regarding the \emph{minimum adversarial perturbation} as below.

The minimum adversarial perturbation is the adversarial perturbation
with the minimum norm.
Formally, given a classifier $c: \sX \rightarrow \sY$ and
a labeled instance $(\vx, y)\in\sX\times\sY$,
the minimum adversarial perturbation is defined as
\begin{align*}
    \vdelta^* = \argmin_{\vdelta}\ \lVert \vdelta \rVert \ \ \text{s.t.} \ c(\vx+\vdelta) \neq y.
\end{align*}


Let $\sD = \{(\vx_i, y_i)\}_{i=1}^N$ be the training dataset,
and $\sD_\sX = \{\vx_i\}_{i=1}^N$ be the training instances without labels.
We have the following theorem on the minimum adversarial perturbation
of the $K$-nearest neighbor classifier ($K$-NN):

\begin{theorem} \label{thm:nn-span}
For every $(\vx, y) \in \sX \times \sY$, there exists $\vw \in \R^N$ such that
the minimum adversarial perturbation of $K$-NN satisfies
\begin{align*}
    \vdelta^* = \sum_{i=1}^N w_i \vx_i.
\end{align*}
In other words, the minimum adversarial perturbation of $K$-NN
is in the subspace $\operatorname{span}(\sD_\sX)$.
\end{theorem}
\begin{proof}
    Given $(\vx, y) \in \sX \times \sY$ and $\sT \subseteq \sD_\sX$ with $\lvert \sT \rvert = K$,
    consider to add a small perturbation $\vdelta \in \sX$
    such that $\sT$ is the $K$ nearest neighbors of $\vx + \vdelta$.
    This problems could be formalized as the following optimization problem:
    \begin{align*} 
    \begin{aligned}
      \min_{\vdelta} \ \
       & \lVert \vdelta \rVert                                                       \\
      \text{s.t.} \
       & \lVert \vx + \vdelta - \vx^+ \rVert \leq \lVert \vx + \vdelta - \vx^-\rVert \\
       & \forall \vx^+ \in \sT,\ \forall \vx^- \in \sD_\sX - \sT.
    \end{aligned}
  \end{align*}
  It is equivalent to the following problem:
  \begin{align*}
    \begin{aligned}
      \min_{\vdelta} \ \
       & \frac{1}{2} \vdelta^\top \vdelta                                                                         \\
      \text{s.t.} \
       & (\vx^- - \vx^+)^\top  \vdelta \leq \frac{1}{2} (\lVert\vx - \vx^-\rVert^2 - \lVert \vx - \vx^+ \rVert^2) \\
       & \forall \vx^+ \in \sT,\ \forall \vx^- \in \sD_\sX - \sT.
    \end{aligned}
  \end{align*}
The constraints could be rewritten in the matrix form:
\begin{align*} 
    \begin{aligned}
      \min_{\vdelta} \ \
       & \frac{1}{2} \vdelta^\top \vdelta \\
      \text{s.t.} \
       & \mA \vdelta \leq \vb.
    \end{aligned}
  \end{align*}
 Obviously, it is a convex quadratic programming problem.
 Let $\vdelta^*_{\sT}$ and $\vlambda^*_{\sT}$ be the optimal points
 of the primal problem and the dual problem respectively.
 By the primal-dual relationship, we have
 \begin{align*}
     \vdelta^*_\sT = -\mA^\top \vlambda^*_\sT.
 \end{align*}
 Considering the form of $\mA$,
 it is obvious that $\vdelta^*_\sT$ is a linear combination of instances in $\sD_\sX$.
 
 Note that the minimum adversarial perturbation $\vdelta^*$ of $K$-NN
 has to be $\vdelta^*_\sT$ for a certain $\sT$.
 Therefore, $\vdelta^*$ has to be in the subspace $\operatorname{span}(\sD_\sX)$.
\end{proof}





Similar results on the minimum adversarial perturbation also hold for
support vector machine (SVM) classifiers~\citep{cortes1995support} as follows:
\begin{theorem} \label{thm:svm-span}
For every $(\vx, y) \in \sX \times \sY$, there exists $\vw \in \R^N$ such that
the minimum adversarial perturbation of SVM satisfies
\begin{align*}
    \vdelta^* = \sum_{i=1}^N w_i \vx_i.
\end{align*}
In other words, the minimum adversarial perturbation of SVM
is also in the subspace $\operatorname{span}(\sD_\sX)$.
\end{theorem}
\begin{proof}
    For simplicity, we only consider the binary case,
    which can be easily extended to the multi-class case
    by strategies such as one-vs-one and one-vs-rest.
    Let $\vw^*$ be the optimal solution of SVM.
    Based on the primal-dual relationship, we have
    \begin{align*}
        \vw^* = \sum_{i=1}^N \alpha_i \vx_i
    \end{align*}
    for some $\valpha \in \R^N$.
    When predicting a perturbed instance, SVM calculates
    \begin{align*}
        \langle\vw, \vx + \vdelta\rangle
         & = \langle\vw,\vx\rangle + \langle\vw,\vdelta\rangle                         \\
         & = \langle\vw,\vx\rangle + \lVert\vw\rVert \lVert\vdelta\rVert \cos(\theta),
    \end{align*}
    where $\theta$ is the angle between $\vw$ and $\vdelta$.
    Therefore, the minimum adversarial perturbation
    that flips the sign of $\langle\vw,\vx\rangle$
    has to be in the direction of $\vw$ with $\cos(\theta)=1$
    or in the opposite direction of $\vw$ with $\cos(\theta)=-1$.
\end{proof}

It is inspiring that $K$-NN and SVM
are very different whereas share the same property:
\begin{center}
    \emph{``Minimum adversarial perturbations prove to be \\
    in the subspace spanned by the training data.''}
\end{center}
Consequently, Theorem~\ref{thm:nn-span} and Theorem~\ref{thm:svm-span} 
motivate us to \emph{search for adversarial perturbations
in the space $\operatorname{span}(\sD_\sX)$},
which is the theoretical foundation of our spanning attack.
Nevertheless, before diving into details of the spanning attack,
it is worth mentioning that
computing minimum adversarial perturbations
for neural networks and tree-based ensemble models
has shown to be NP-hard~\citep{katz2017reluplex,kantchelian2016evasion},
and it is an open problem in what conditions minimum adversarial perturbations
for neural networks and tree-based ensemble models
are also in the space $\operatorname{span}(\sD_\sX)$.

In practice, it is not reasonable to assume the training data are
available to attackers.
To make a relaxation, we assume that the attacker only has access to
an auxiliary unlabeled dataset $\sS$.
By this means, subspace attackers search for adversarial perturbations
in $\operatorname{span}(\sS)$,
namely the \textbf{\emph{spanning attack}}, i.e.,
\emph{the subspace attack by spanning an auxiliary unlabeled dataset}.
For convenience, we simply term
the auxiliary unlabeled dataset as the \emph{subspace dataset}.

By Corollary~\ref{cor:transform},
given a subspace dataset $\sS$,
the spanning attack requires a set of orthonormal vectors which is a basis for $\operatorname{span}(\sS)$
so as to transform a random attack into a subspace attack.
We could make it by the standard process of \emph{orthonormalization},
which can be performed by the {Gram-Schmidt process},
the {Householder transformation} etc~\citep{cheney2010linear}.
Therefore, the overall procedures of our spanning attack is as below
(also shown in Figure~\ref{fig:workflow}):
\begin{enumerate}
    \item Compute a basis of $\sS$ by orthonormalization;
    \item Transform the random attack into a subspace attack by Algorithm~\ref{alg:sample};
    \item Attack the target model with the resulting subspace attack.
\end{enumerate}

\subsection{Selective spanning attack} \label{sec:selective}

In this section, we talk about an extension of the spanning attack.
The spanning attack searches for adversarial perturbations
in the space $\operatorname{span}(\sS)$,
which is a subspace of the input space $\sX$.
A natural question is
\emph{whether it is possible to benefit more by means of
    explicitly selecting a subspace of $\operatorname{span}(\sS)$
    instead of using $\operatorname{span}(\sS)$ directly}.
We term the method which searches for adversarial perturbations
in a non-trivial subspace
of $\operatorname{span}(\sS)$ as the \emph{selective spanning attack},
as it selects a subspace from $\operatorname{span}(\sS)$.

In the case of the selective spanning attack,
the Gram-Schmidt process or Householder transformation is not instructive 
to select a subspace of $\operatorname{span}(\sS)$,
since there is no significant difference
among the derived orthonormal vectors.

Instead, we employ the singular value decomposition (SVD) to derive a set of orthonormal vectors
which is a basis of $\operatorname{span}(\sS)$.
In particular,
assume the subspace dataset has $M^\prime$
different instances
$\sS = \{\vx_1, \vx_2,\ldots, \vx_{M^\prime}\}$ and
$\mS \in \sR^{M^\prime\times D}$ is the matrix of which the $j$-th row is $\vx_j^\intercal$.
(We use $M^\prime$ here because $M$ denotes the number of orthonormal vectors 
as Algorithm~\ref{alg:sample}.)
By SVD, $\mS$ can be decomposed into the form
\begin{align*}
    \mS = \mU \mSigma \mV^\intercal,
\end{align*}
where $\mU\in\sR^{M^\prime\times M^\prime}$ and $\mV\in \sR^{D\times D}$ are orthogonal matrices
and $\mSigma \in \sR^{M^\prime\times D}$ is a diagonal matrix,
of which diagonal entries are singular values.

It could be proved that the right singular vectors (columns of $\mV$) satisfy the following property:
\begin{lemma} \label{lm:svd}
    Right singular vectors of which the corresponding singular values are larger than zero
    are an orthonormal basis for $\operatorname{span}(\sS)$.
\end{lemma}
\begin{proof}
    Let ${M}$ denote the number of non-zero singular values.
    Then, we have the compact SVD as
    \begin{align*} \label{eq:compact}
        \mS = \mU_{M} \mSigma_{M} \mV_{M}^\intercal.
    \end{align*}
    Let $\ve_i$ denote the $i$-th column of $\mV_{M}$.
    The objective is to prove
    \begin{align*}
        \operatorname{span}(\sS) = \operatorname{span}(\ve_1, \ve_2, \ldots, \ve_{M}),
    \end{align*}
    which is equivalent to
    \begin{itemize}
        \item (i) $\operatorname{span}(\sS) \subseteq \operatorname{span}(\ve_1, \ve_2, \ldots, \ve_{M})$ and
        \item (ii) $\operatorname{span}(\ve_1, \ve_2, \ldots, \ve_{M}) \subseteq \operatorname{span}(\sS)$.
    \end{itemize}
    
    For any $\va\in\operatorname{span}(\sS)$,
    by definition there exists $\vb\in\sR^{M^\prime}$ such that
    \begin{align*}
        \va = \vb^\intercal \mS = (\vb^\intercal\mU_{M}\mSigma_{M}) \mV_{M}^\intercal.
    \end{align*}
    Therefore, we have (i)
    $\operatorname{span}(\sS) \subseteq \operatorname{span}(\ve_1, \ve_2, \ldots, \ve_{M})$.
    
    By the compact SVD, we also have
    \begin{align*}
        \mSigma^{-1}_{M} \mU_{M}^\intercal \mS = \mV_{M}^\intercal.
    \end{align*}
    Thus, for any $\va\in\operatorname{span}(\ve_1, \ve_2, \ldots, \ve_{M})$,
    there exists $\vb\in\sR^{M}$ such that
    \begin{align*}
        \va = \vb^\intercal \mV_{M}^\intercal =
        (\vb^\intercal \mSigma_{M}^{-1} \mU_{M}^\intercal) \mS.
    \end{align*}
    Therefor, we have (ii)
    $\operatorname{span}(\ve_1, \ve_2, \ldots, \ve_{M}) \subseteq \operatorname{span}(\sS)$.
    
\end{proof}

We denote these right singular vectors as $\ve_1$, $\ve_2$, $\ldots$, $\ve_{M}$
with corresponding singular values larger than zero,
and they are sorted according to the corresponding singular values
such that the singular values have
$\sigma_1 \geq \sigma_2 \geq \ldots \geq \sigma_{M}$.
Then, we have roughly two options for the selective spanning attack:
selecting the top singular vectors,
and selecting the bottom singular vectors.
We term these two options as the \emph{top spanning attack}
and the \emph{bottom spanning attack} respectively.


\begin{figure}[t]
    \centering
    \subfloat[]{\label{fig:tsa}
        \includegraphics[width=.13\textwidth]{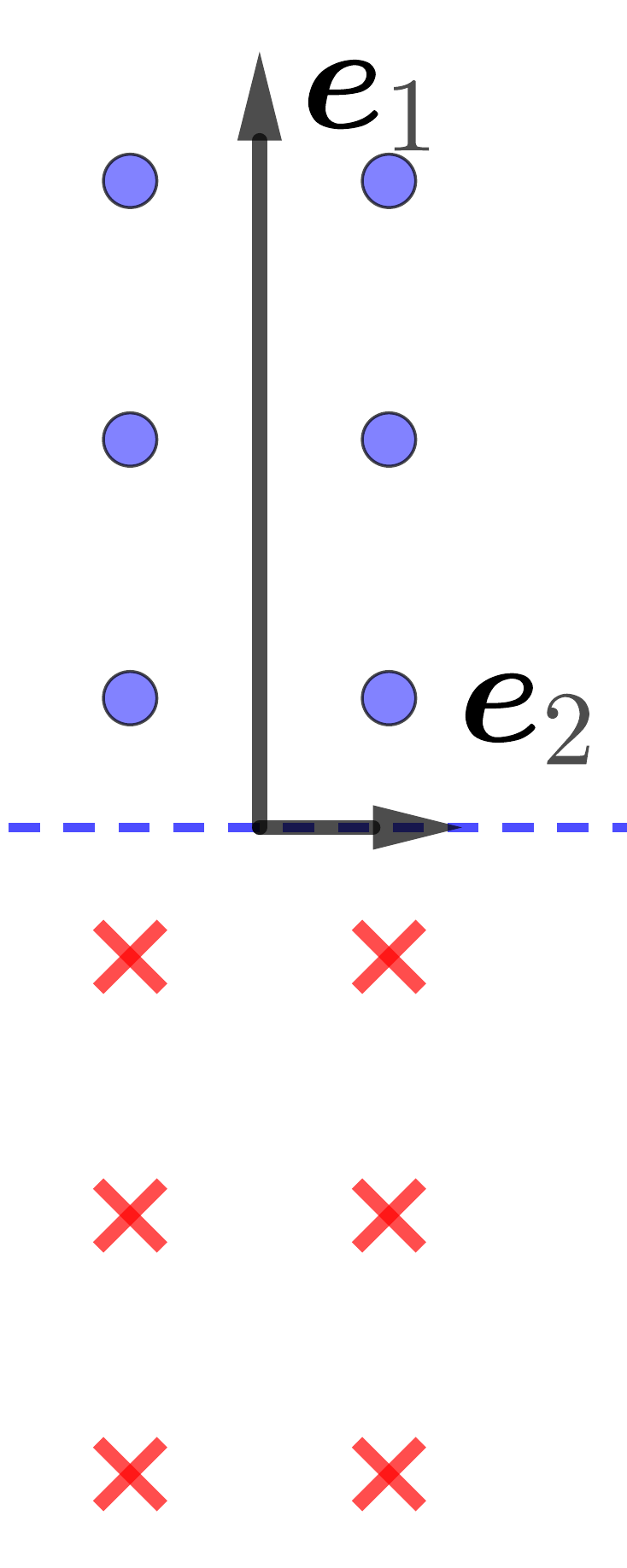}}
    \quad
    \quad
    \quad
    \quad
    \quad
    \subfloat[]{\label{fig:bsa}
        \includegraphics[width=.13\textwidth]{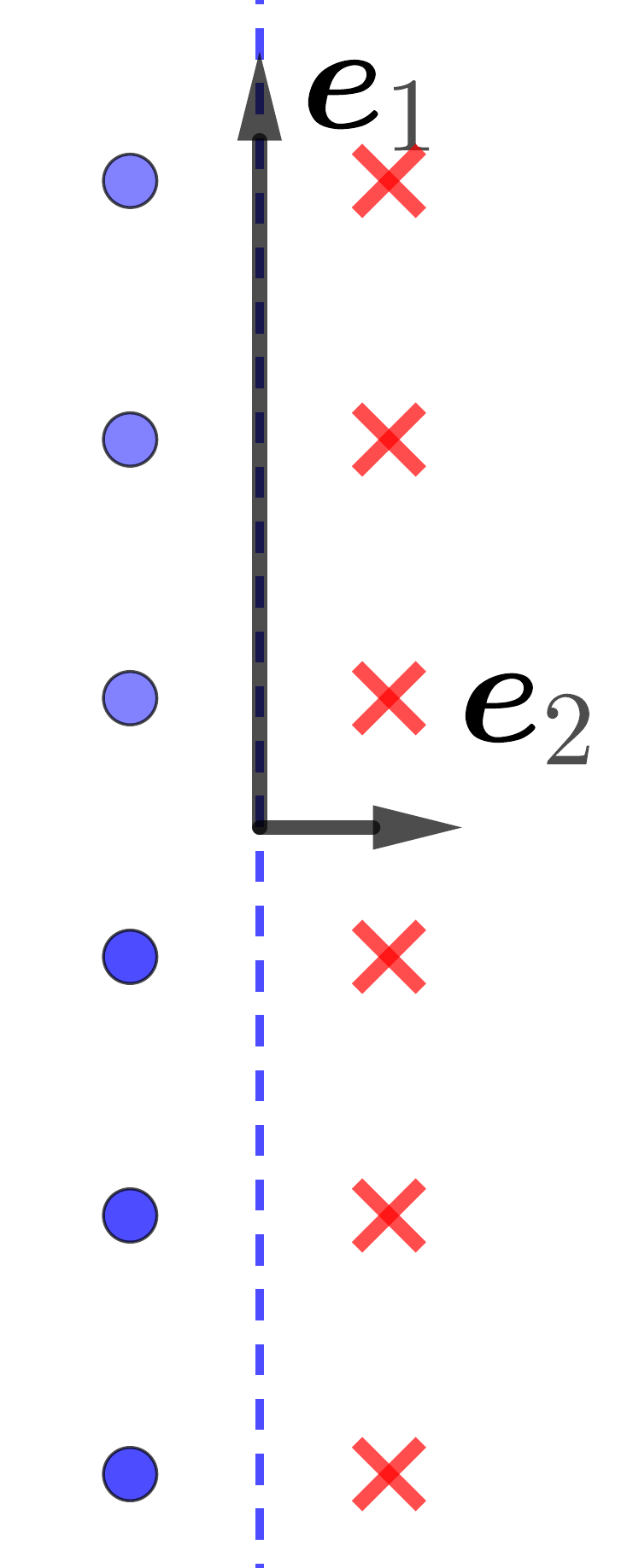}}
    
    \caption{Illustration of top and bottom spanning attacks.
        The only difference between Figure~\ref{fig:tsa} and Figure~\ref{fig:bsa} is the ground-truth labels.
        $\ve_1$ and $\ve_2$ are right singular vectors,
        and their lengths represent the corresponding singular values.
        The top spanning attack, i.e., selecting $\operatorname{span}(\ve_1)$,
        is better than the bottom spanning attack,
        i.e., selecting $\operatorname{span}(\ve_2)$, in Figure~\ref{fig:tsa};
        whereas Figure~\ref{fig:bsa} is the opposite case.
    }
    \label{fig:selective}
\end{figure}

The top spanning attack and the bottom spanning attack
have their own advantageous situations
depending on the labels of the underlying data distribution.
We illustrate two toy cases in Figure~\ref{fig:selective}.
In the first case, the top spanning attack is favorable
since we can find adversarial perturbations along $\ve_1$,
and the second case is the exact opposite.
Roughly speaking, top singular vectors represent directions along the manifold of the dataset,
and bottom singular vectors represent directions out of the manifold.
It is believed that for high-dimensional datasets
adversarial examples widely exist in directions out of the manifold \citep{stutz2019disentangling}.
Therefore, the bottom spanning attack would be a better choice in practice,
which is validated in our experiments.

\section{Experiments} \label{sec:exp}

In this section, we empirically validate the performance of
the proposed spanning attack.
Specifically, we select four representative black-box (random) attacks as baselines
and employ the spanning attack method to reinforce them:
\begin{itemize}
    \item The RGF attack~\citep{cheng2019improving}: a soft-label attack with Gaussian random vectors
          within the framework considered in the case study for soft-label attacks;
    \item The SPSA attack~\citep{uesato2018adversarial}: a soft-label attack similar to the RGF attack
    but with Rademacher random vectors instead of Gaussian random vectors;
    (In our implementation,
    SPSA has all hyper-parameters the same
    as RGF except the distribution for random vectors.)
    \item The boundary attack~\citep{brendel2018decision}: a pioneering widely-used hard-label attack based on random walks;
    \item The Sign-OPT attack~\citep{cheng2020signopt}: a state-of-the-art hard-label attack based on direction estimation.
\end{itemize}



We perform untargeted black-box attacks on the ImageNet dataset \citep{deng2009imagenet}.
Attacks are performed against the pre-trained ResNet-50~\citep{he2016deep}, VGG-16~\citep{simonyan2015very}
and DenseNet-121~\citep{huang2017densely}
from the PyTorch model zoo~\citep{steiner2019pytorch},
since these architectures are diverse and representative,
and many real-world deployed models are based on them.
Correctly-classified images
are randomly sampled from the validation set as the evaluation dataset,
of which every labeled image is the instance to attack.
The size of the evaluation dataset for soft-label attacks is 1,000,
and the one for hard-label attacks is 100 for computational efficiency.
Another 1,000 unlabeled images are sampled from the validation set as the subspace dataset.
The evaluation dataset and the subspace dataset are mutually exclusive.
We set the perturbation radius $\epsilon = \sqrt{0.001 D}$ and the budget $B = 10,000$
by convention~\citep{cheng2019improving}.
(We also study whether the radius parameter $\epsilon$ impacts performance in Section~\ref{sec:radii}.)
If an attack method finds an adversarial perturbation $\vdelta$ within $B$ queries
such that $\lVert\vdelta\rVert\leq\epsilon$ holds,
then this attack is \emph{successful};
otherwise it is failed.
Therefore, we have two criteria for a black-box attack:
(i) whether it is successful and (ii) the number of queries it executes when successful.

All hyper-parameters of the spanning attacks are the same as the corresponding baselines.
The only difference is introducing an appropriate subspace via our methods.
We refer to \citet{cheng2019improving}, \citet{uesato2018adversarial},
\citet{brendel2018decision} and \citet{cheng2020signopt} for details of
the baseline black-box methods.

\subsection{Main results}
\label{sec:main}

\begin{table}[t]
    \centering
    \caption{
        Comparison between the baseline black-box attacks and the resulting spanning attacks.
    }
    \resizebox{\linewidth}{!}{
    \begin{tabular}{@{}lllrrr@{}}
        \toprule
        & & & Success rate & Query mean & Query median \\ \midrule
        \multirow{9}{*}{ResNet-50} & \multirow{2}{*}{RGF (soft-label)}      & Baseline        & 0.971 & 589.575 &	358.0      \\
        &                                       & Spanning attack & 0.991 &	329.541 &	205.0      \\ \cmidrule{2-6}
        & \multirow{2}{*}{SPSA (soft-label)} & Baseline        & 0.972 &	584.772	& 358.0 \\
        &                                       & Spanning attack & 0.991	& 330.725	& 205.0 \\ \cmidrule{2-6}
        & \multirow{2}{*}{Boundary (hard-label)} & Baseline        & 0.720        & 4133.903   & 3291.0     \\
        &                                       & Spanning attack & 0.880        & 3197.557   & 2569.5     \\ \cmidrule{2-6}
        & \multirow{2}{*}{Sign-OPT (hard-label)} & Baseline        & 0.970        & 2392.175   & 2143.0     \\
        &                                       & Spanning attack & 1.000        & 1053.220   & 647.0      \\
        \midrule
        \multirow{9}{*}{VGG-16} & \multirow{2}{*}{RGF (soft-label)}      & Baseline        & 0.966 &	389.519 &	256.0      \\
        &                                        & Spanning attack & 0.975	& 261.335	& 154.0      \\ \cmidrule{2-6}
        & \multirow{2}{*}{SPSA (soft-label)} & Baseline        & 0.968	& 386.187	& 256.0 \\
        &                                       & Spanning attack & 0.975	& 262.905	& 154.0 \\ \cmidrule{2-6}
        & \multirow{2}{*}{Boundary (hard-label)} & Baseline        & 0.810        & 3467.086   & 2787.0     \\ 
        &                                        & Spanning attack & 0.940        & 2972.755   & 2263.0     \\ \cmidrule{2-6}
        & \multirow{2}{*}{Sign-OPT (hard-label)} & Baseline        & 1.000        & 1665.080   & 1450.0     \\ 
        &                                        & Spanning attack & 1.000        & 840.900    & 572.5      \\ 
        \midrule
        \multirow{9}{*}{DenseNet-121} & \multirow{2}{*}{RGF (soft-label)}      & Baseline        & 0.981 &	528.312 &	358.0      \\
        &                                        & Spanning attack & 0.995 &	272.043 &	154.0      \\ \cmidrule{2-6}
        & \multirow{2}{*}{SPSA (soft-label)} & Baseline        & 0.984	& 552.982	& 358.0 \\
        &                                       & Spanning attack & 0.997	& 299.941	& 154.0 \\ \cmidrule{2-6}
        & \multirow{2}{*}{Boundary (hard-label)} & Baseline        & 0.670        & 3806.687   & 3389.0     \\
        &                                        & Spanning attack & 0.890        & 3063.449   & 2261.0     \\ \cmidrule{2-6}
        & \multirow{2}{*}{Sign-OPT (hard-label)} & Baseline        & 0.980        & 2407.398   & 1863.5     \\
        &                                        & Spanning attack & 1.000        & 1014.280   & 688.5      \\ 
        \bottomrule
    \end{tabular}
    }
    \label{tab:main}
\end{table}

\begin{figure}[t]
    \centering
    \includegraphics[width=\linewidth]{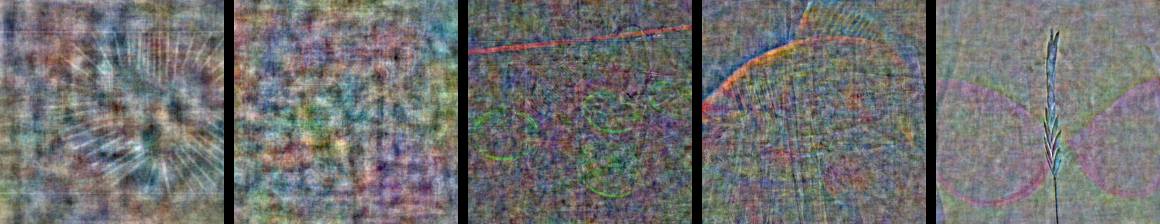}
    \caption{Visualization for vectors of the orthonormal basis.}
    \label{fig:basis}
\end{figure}

\begin{figure}[t]
    \centering
    \includegraphics[width=\linewidth]{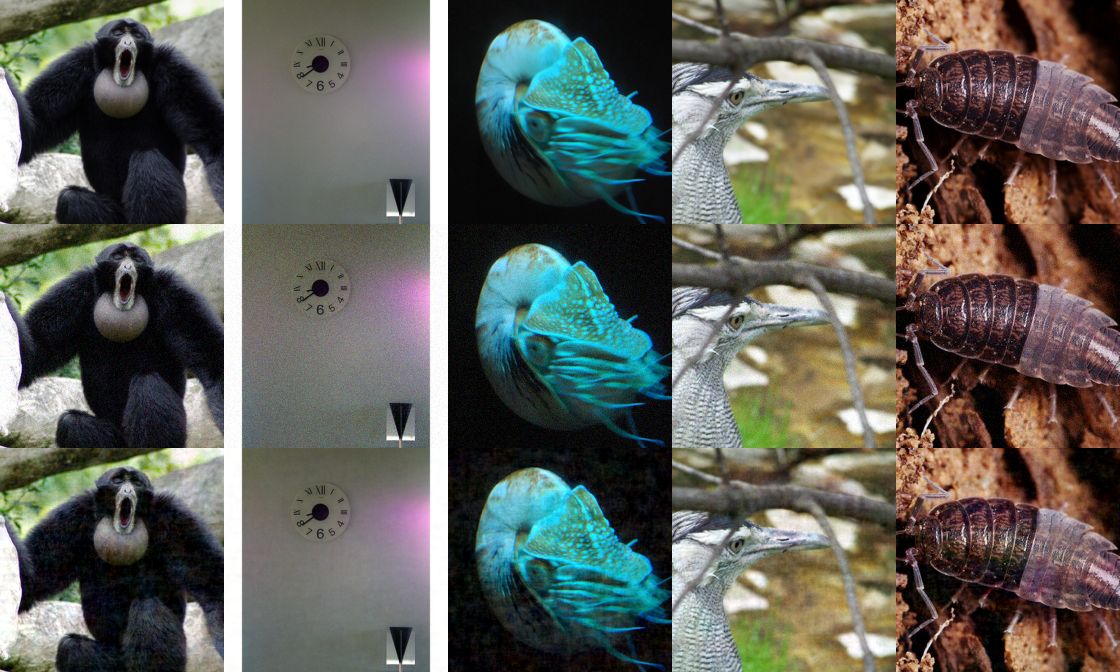}
    \caption{Examples of the adversarial images.
        The first row is the original images;
        the second row is the adversarial images crafted by the baseline attack (Sign-OPT against ResNet-50);
        the third row is the adversarial images crafted by the corresponding spanning attack.}
    \label{fig:examples}
\end{figure}


Success rates, query means and query medians on the evaluation dataset are reported in
Table~\ref{tab:main}.
By convention only successful adversarial perturbations are counted for query means and query medians.
On the one hand, this criterion is favorable for the method
with a lower success rate and a lower query number on successful adversarial perturbations.
On the other hand and more importantly,
if a method has a higher success rate and a lower query number on successful perturbations than the other,
we would have sufficient confidence to conclude that the first method performs better.

Our results show that the spanning attack method reinforces
the baselines significantly
in terms of both success rates and query numbers,
consistently across \emph{all} of the baseline methods and \emph{all} of the pre-trained target models.

In particular,
in the case of the RGF attack and the Sign-OPT attack,
the spanning attack only needs approximately \emph{half} the queries of the baseline
for a successful attack,
and increases success rates in the meantime.
For example, in the Sign-OPT case against ResNet-50,
the spanning attack improves the success rate to $100\%$,
and more crucially, it only requires $44\%$ queries in terms of the query mean
and $30\%$ queries in terms of the query median!

In the case of the boundary attack,
while success rates of the baseline attack
within the given budget
is not satisfying,
our spanning attack improves the success rates
by a wide margin.
For example, in the Boundary case against VGG-16,
the spanning attack improves the success rate from $81\%$ to $94\%$.

\paragraph{Visualization of the subspace basis.}
A sample of vectors of the resulting orthonormal basis
are visualized in Figure~\ref{fig:basis}.
Note that these vectors reflect low-dimensional structures of the subspace
rather than white Gaussian noise in the input space.

\paragraph{Examples of adversarial images.}
Several adversarial images, crafted by the baseline method and 
the spanning attack, are displayed in Figure~\ref{fig:examples}.
All these adversarial images does not show any significant difference from the 
original images due to the fact that they have the same constraint on the perturbation norm.

\paragraph{Comments on improvement of success rates.}
The black-box attack is a non-convex zeroth-order optimization problem;
there is always a chance that the baseline method is trapped in some local areas,
and as a consequence fails to attack.
For instance, when RGF estimates gradients,
informative random vectors could be too sparse to find an accurate gradient,
as these random vectors are sampled from a large space.
In contrast, the spanning attack employs prior knowledge (encoded in the subspace)
about adversarial perturbations, 
and hence reduces the possibility of being trapped.
That's why the spanning attack could improve success rates and query efficiency simultaneously.
It is noteworthy that the capability of the spanning attack depends on
the subspace dataset (the prior knowledge encoded);
we will carefully investigate it in the next section.

Since the results are consistent across all baseline methods and target models,
in the following we take the RGF attack against ResNet-50 as illustration by default
to avoid unnecessary repetition.


\subsection{Investigation on the subspace}

In this section, we study to what extent 
the subspace impacts on the performance of the spanning attack, 
and furthermore how we could establish the subspace dataset in practice.


\begin{figure}[t]
    \centering
    \subfloat[Success rate]{\label{fig:success}
        \includegraphics[width=.32\textwidth]{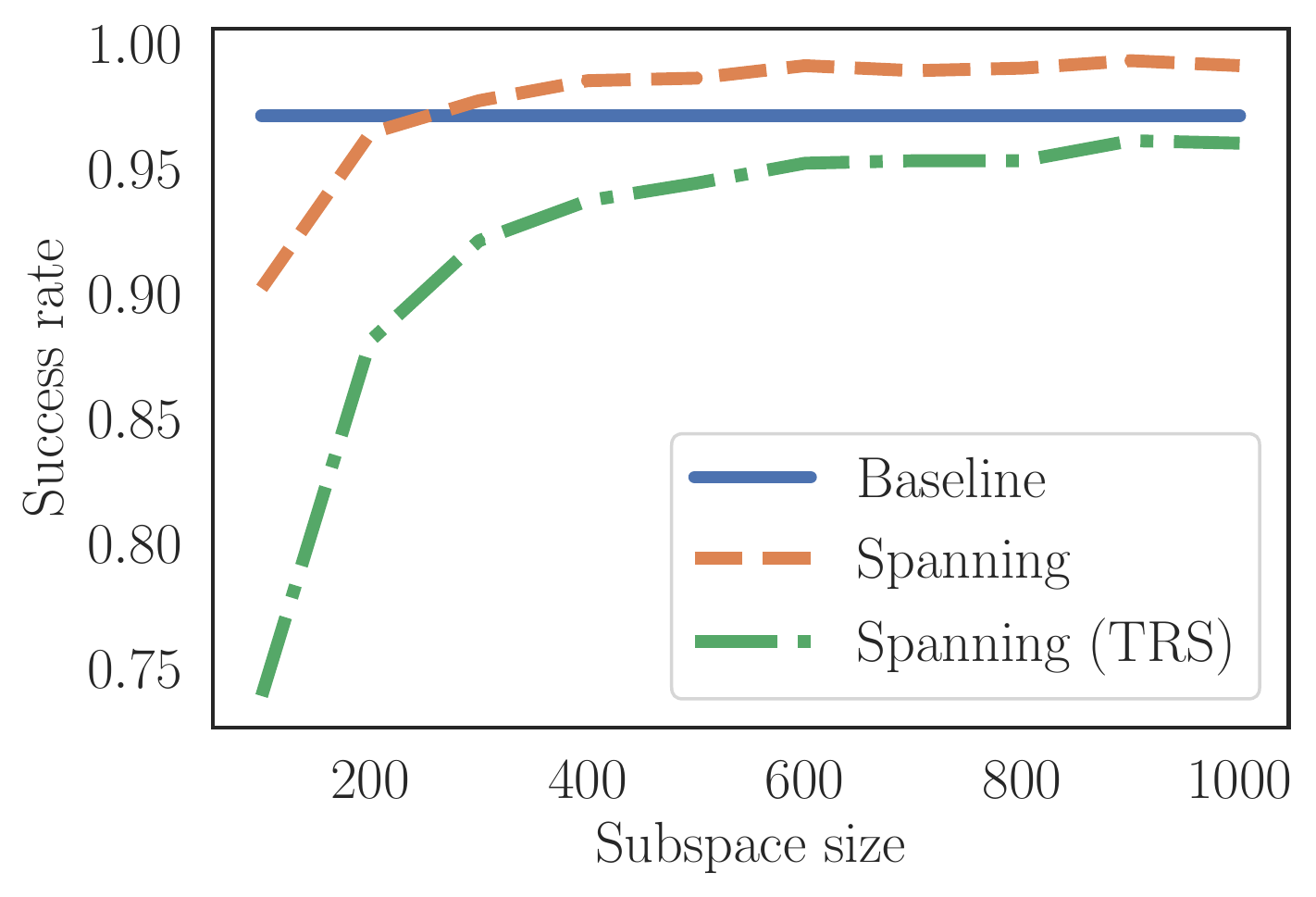}}
    \subfloat[Query mean]{\label{fig:mean}
        \includegraphics[width=.32\textwidth]{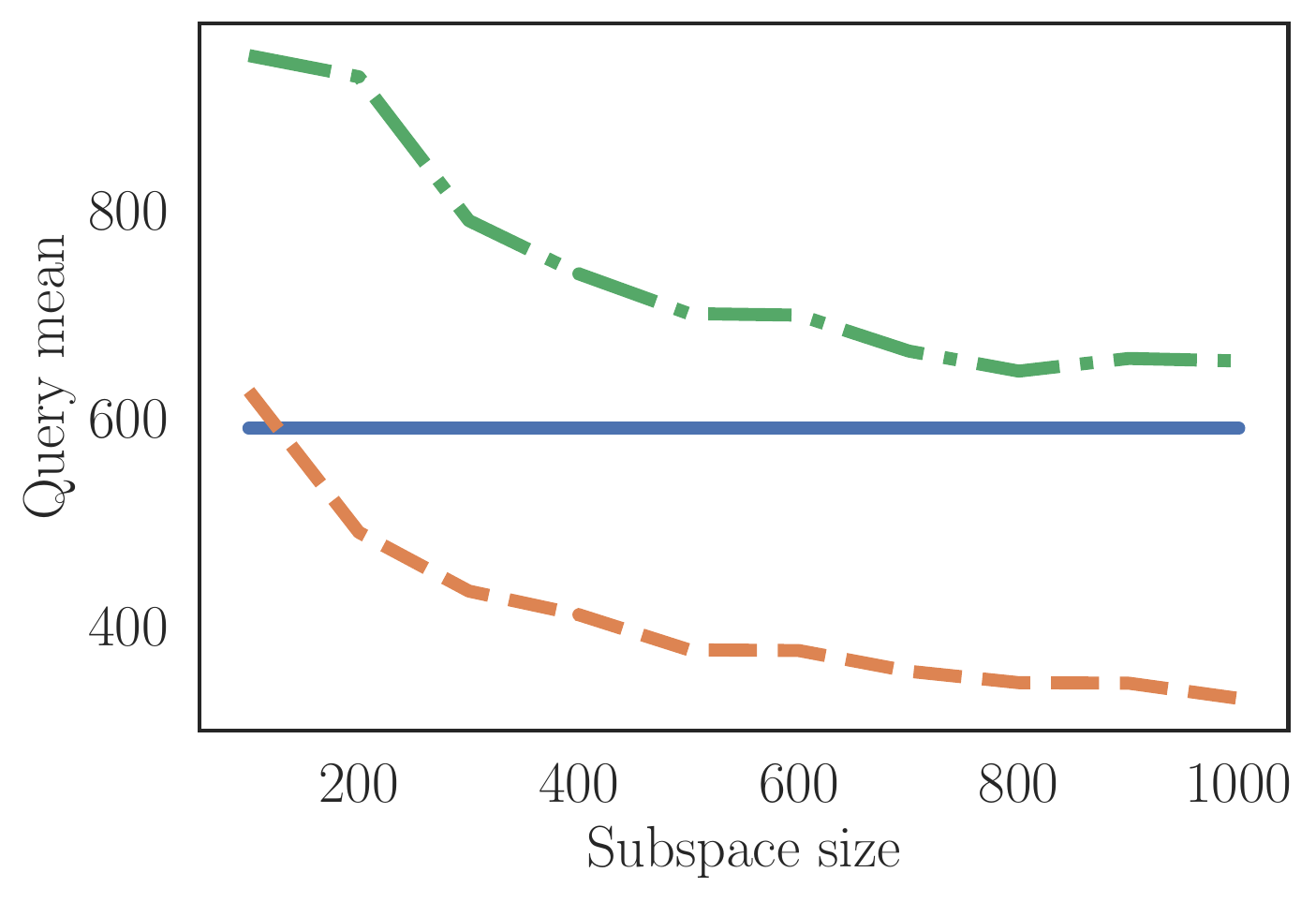}}
    \subfloat[Query median]{\label{fig:median}
        \includegraphics[width=.32\textwidth]{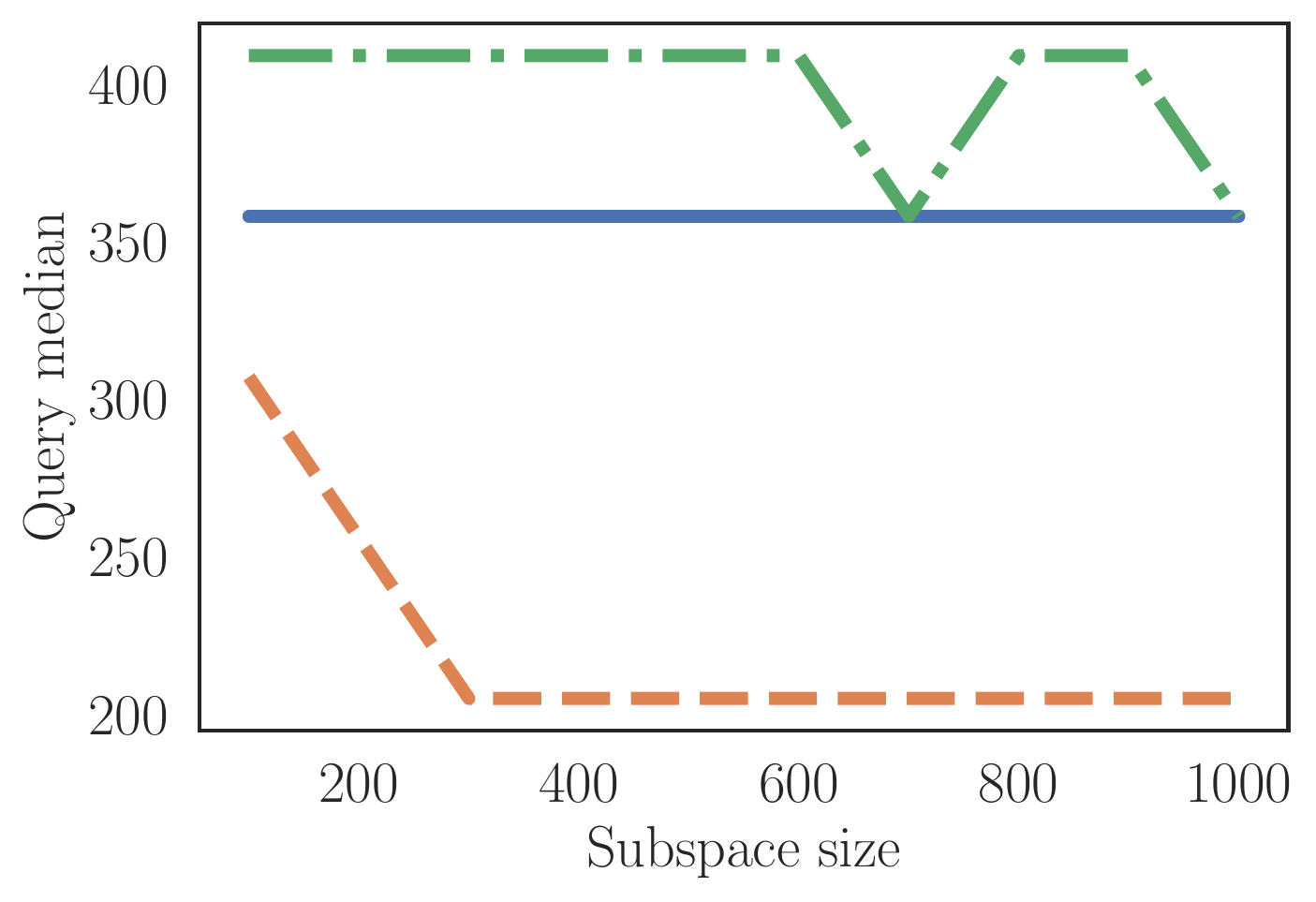}}
    
    \caption{Attack performance with different sizes of the subspace.
    TRS stands for Totally Random Subspace.}
    \label{fig:size}
\end{figure}

\subsubsection{Size of the subspace dataset}
\label{sec:exp:size}

We show attack performance with different sizes of the subspace dataset
in Figure~\ref{fig:size}.
(We will talk about TRS in Section~\ref{sec:exp:distribution}.)
In our experiments, the minimum size is 100 and the maximum size is 1,000.
In this scope,
the larger the subspace size,
the better the performance of the spanning attack.
In contrast, the baseline method is the extreme case 
where the subspace size is $D = 224 \times 224 \times 3 = 150,528$.
Therefore, it is expected that the performance of the spanning attack
will reach the peak and then slide down as the subspace size increases.
It is noteworthy that \emph{even a small subspace dataset,
$\approx 400$ as shown in Figure~\ref{fig:size},
would help the spanning attack defeat the baseline methods.}


\subsubsection{Distribution of the subspace dataset}
\label{sec:exp:distribution}

\begin{table}[t]
    \centering
    \caption{Results of the spanning attack with label-biased subspace datasets,
    spanning attack with Flickr8k subspace datasets,
        spanning attack with totally random subspace (a subspace without any prior),
        bottom spanning attack and top spanning attack.
        TRS stands for Totally Random Subspace.
    }
    \begin{tabular}{@{}lrrr@{}}
        \toprule
                                   & Success rate & Query mean & Query median \\ \midrule
        Baseline                   & 0.971 & 589.575 &	358.0        \\
        Spanning attack            & 0.991 &	329.541 &	205.0        \\ \midrule
        Spanning attack (label biased) & 0.991	& 316.572	& 205.0 \\ 
        Spanning attack (Flickr8k) & 0.990 &	318.333 &	205.0        \\
        Spanning attack (TRS)    & 0.960        & 654.491    & 358.0        \\ \midrule
        Bottom spanning attack     & 0.991 &	298.817	 & 154.0        \\
        Top spanning attack        & 0.991 &	354.346 &	205.0        \\ \bottomrule
    \end{tabular}
    \label{tab:more}
\end{table}

We investigate whether it is necessary to sample the subspace dataset
from the same distribution as the training data.
Specifically, we further try three settings for the subspace dataset:
\begin{itemize}
    \item 1) instances of top 50 classes
    from the ImageNet validation set.
    Note that there are 1,000 classes in total, and hence it is a \emph{label-biased} setting.
    \item 2) instances from the Flickr8k dataset~\citep{hodosh2015framing},
    which is much different from the ImageNet dataset.
    \item 3) instances sampled from a uniform distribution.
    In other words, the spanned subspace could be seen as a \emph{totally random subspace} (TRS) without any prior knowledge.
\end{itemize}



The results are displayed in the middle area of Table~\ref{tab:more}.
On the one hand,
the results of the label-biased and the Flickr8k spanning attack are still better than the baseline,
and competitive with the spanning attack in Section~\ref{sec:main}
(see the upper area of Table~\ref{tab:more} for convenience),
where the subspace dataset is sampled i.i.d.\ (without any bias) from the ImageNet validation set.
It suggests that
\emph{even a biased subspace dataset suffices to work},
which extends the application range of the spanning attack.
In a word,
the subspace dataset does not necessarily
has to be sampled from the same distribution
with the training data.

On the other hand, the spanning attack with a totally random subspace performs
even worse than the baseline.
(To better illustrate this issue, we also show attack performance 
for different subspace sizes 
with totally random subspaces
in Figure~\ref{fig:size}.)
In other words, the totally random subspace plays a \emph{negative} role on performance.
The result validates that
\emph{prior knowledge given by the subspace dataset is necessary,
rather than an arbitrary low-dimensional subspace.}

\paragraph{Discussion on establishing the subspace dataset in practice.}
The experimental results
of Section~\ref{sec:exp:size} and Section~\ref{sec:exp:distribution}
jointly suggest that
the conditions which the subspace dataset has to obtain
is not too strict in practice:
the size of the subspace dataset could be very small,
and the distribution of the subspace could be different from the one of the training data.
Therefore, when applying the subspace attack method in real-world applications,
we only need to collect a \emph{small} set of \emph{unlabeled} data \emph{related} to the target model. 
For instance, if the task is attacking a face recognition system,
one possibility is to crawl the web and find some face pictures in advance.

\subsubsection{Bottom and top spanning attack}

We investigate whether the selective spanning attack could further improve performance.
In our experiments,
the bottom 800 singular vectors (remember the total number of the singular vectors is 1,000)
are used for the bottom spanning attack,
and the top 800 singular vectors are used for the top spanning attack.
The comparison among the original spanning attack,
bottom spanning attack and top spanning attack are shown in the lower area of Table~\ref{tab:more}.
The results show that the bottom spanning attack could further improve performance,
whereas the top spanning attack has a negative impact.
This is an empirical validation that adversarial perturbations are more likely
to appear in directions out of the data manifold, rather than along the data manifold,
as discussed in Section~\ref{sec:selective}.

\subsection{Sensitivity to radii} \label{sec:radii}

We investigate whether the given radius affects
the capability of the spanning attack over the baseline method.
We report success rates, query means and query medians with varying radii.
The results are illustrated in Figure~\ref{fig:sensitivity},
and show that the spanning attack improves the baseline method consistently across different radii.
Note that in all the other experiments the radius is set $\epsilon = \sqrt{0.001 D} \approx 12.27$.

\begin{figure}[t]
    \centering
    \subfloat[Success rate]{\label{fig:sen_success}
        \includegraphics[width=.32\textwidth]{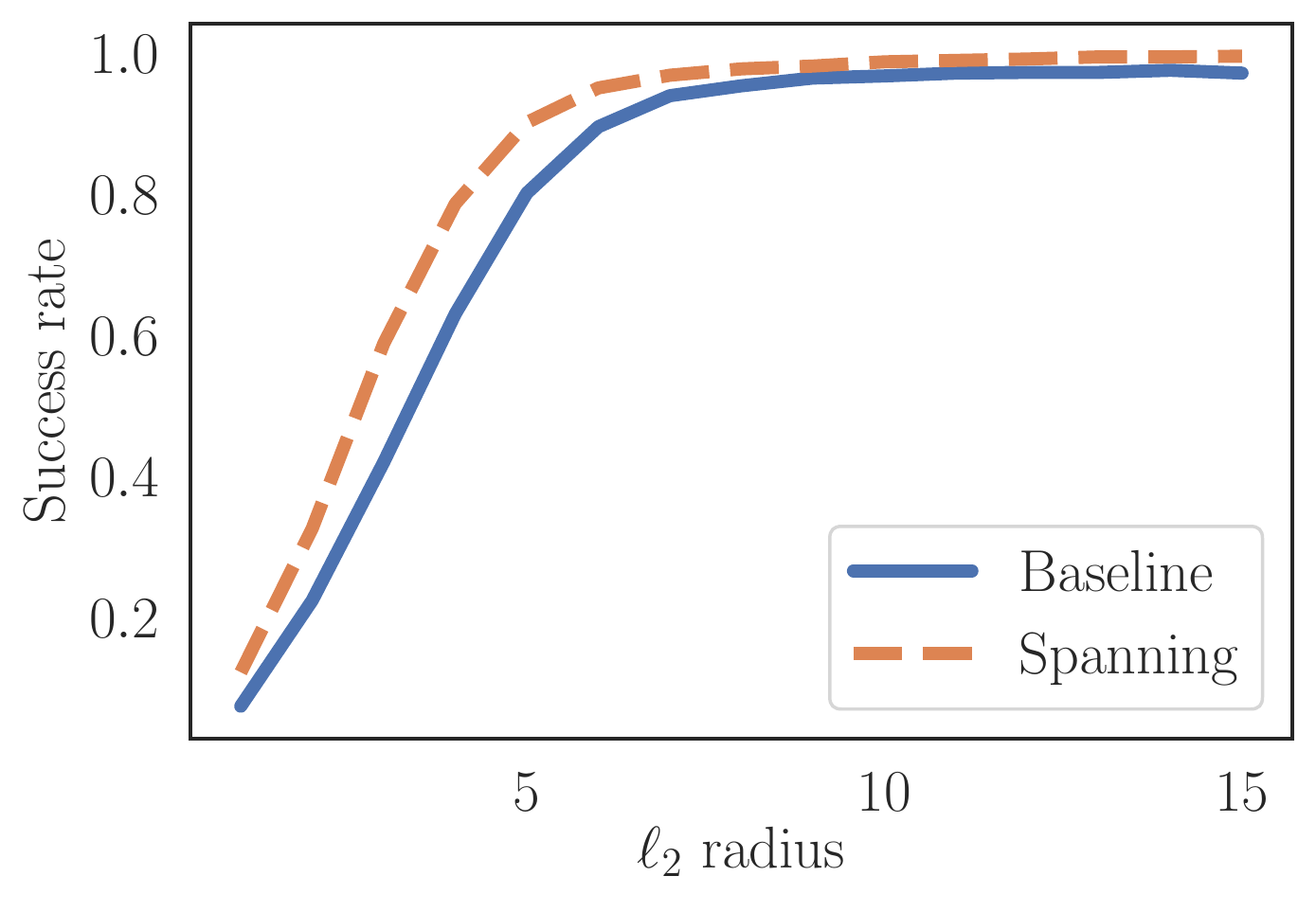}}
    \subfloat[Query mean]{\label{fig:sen_mean}
        \includegraphics[width=.32\textwidth]{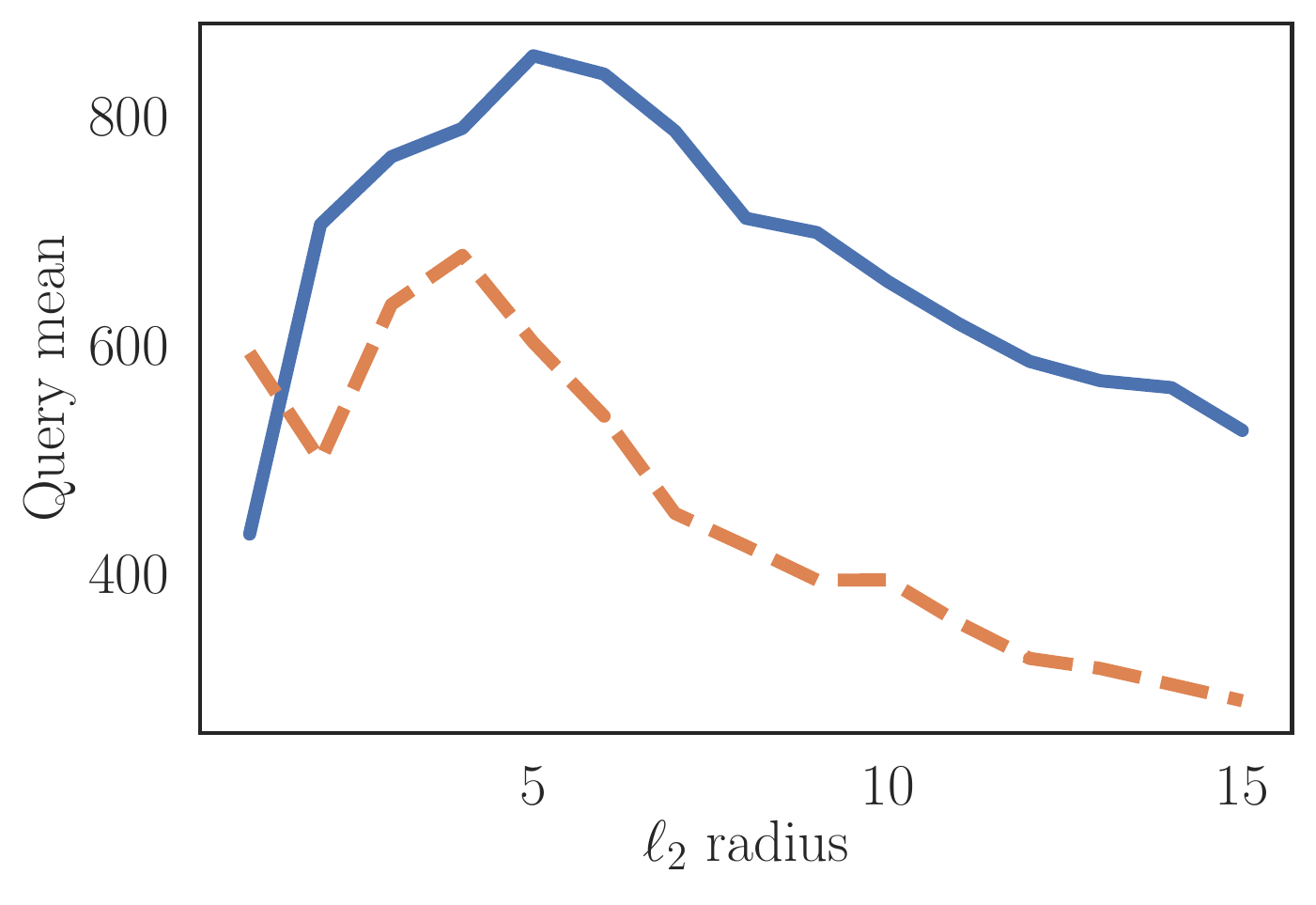}}
    \subfloat[Query median]{\label{fig:sen_median}
        \includegraphics[width=.32\textwidth]{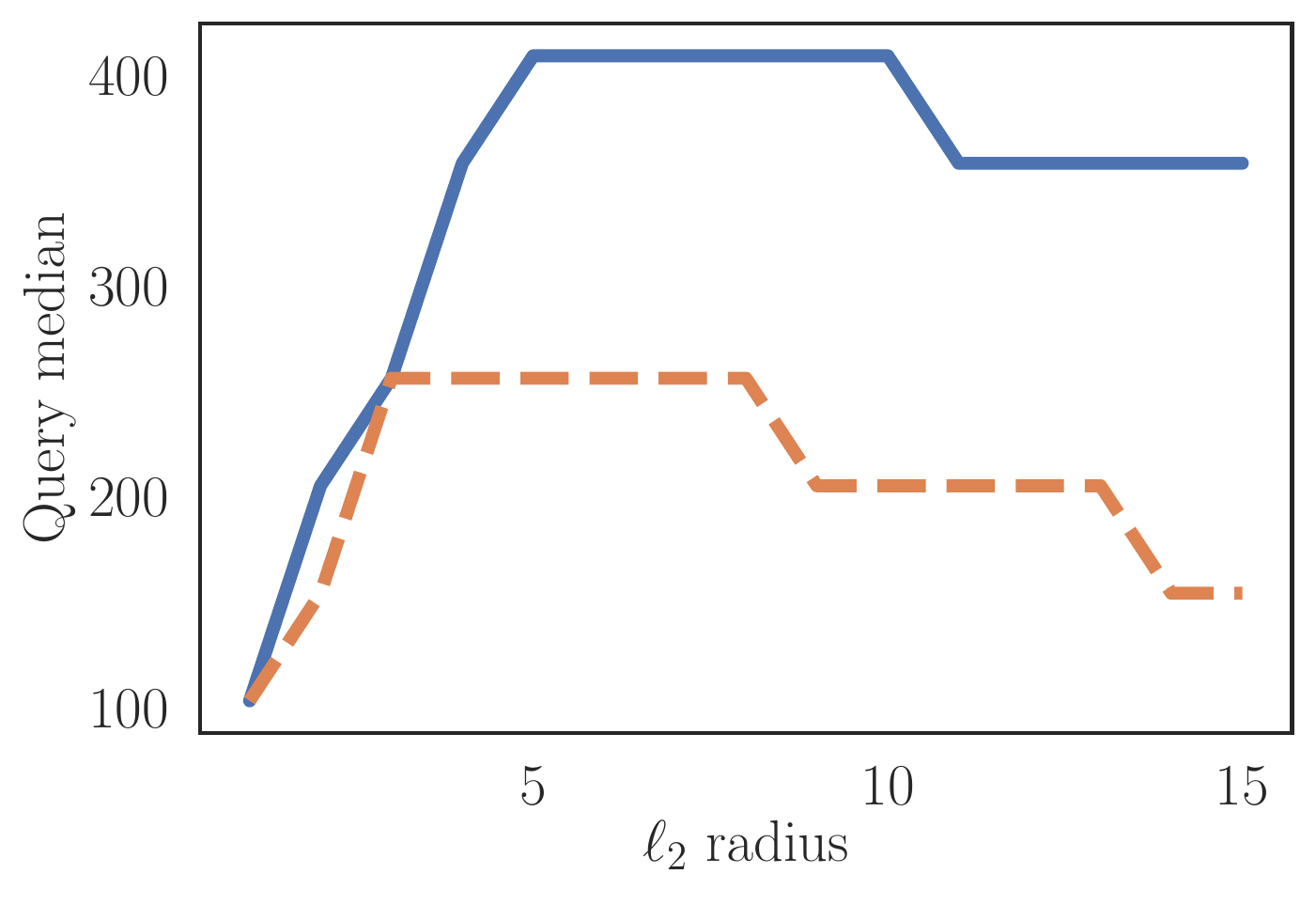}}
    \caption{Attack performance with different radii.}
    \label{fig:sensitivity}
\end{figure}

\subsection{More discussion with related work}

Although the work of \citet{guo2019subspace} has a setting different from ours
as discussed in Section~\ref{sec:related},
for completeness we try to adapt their methods for comparison.
They require an auxiliary \emph{labeled} dataset,
focus on $\ell_\infty$ norm and only considers the soft-label black-box attack.
In order to have a comparison,
we let the subspace dataset be \emph{labeled} with size 1,000.
We notice that \emph{with such a small subspace dataset} in our setting,
\citet{guo2019subspace}'s method does not perform well.
For instance, when attacking ResNet-50, it has the success rate 58.7\% and the query mean 641.283.
The results for attacking VGG-16 and DenseNet-121 are similar
(VGG-16: success rate 68.6\%, query mean 558.044; DenseNet-121: success rate 59\%, query mean 623.603).
It is primarily due to the fact that their method trains substitute models with labeled data.
As a consequence, when the dataset is too small,
it is difficult to train reliable substitute models.


\section{Conclusion} \label{sec:conclusion}

We propose a general technique named the spanning attack
to improve efficiency of black-box attacks.
The spanning attack is motivated by the theoretical analysis
that minimum adversarial perturbations
of machine learning models incline to be in the subspace
of the training data.
In practice, the spanning attack only requires a small auxiliary unlabeled dataset,
and is applicable to a wide range of black-box attacks
including both the soft-label black-box attacks and hard-label black-box attacks.
Our experiments show that the spanning attack can significantly
improve the query efficiency and success rates of black-box attacks simultaneously.

\bibliographystyle{plainnat}

\bibliography{ref}
\end{document}